\documentclass[11pt]{article}
\usepackage{fullpage}

\usepackage{amsmath}
\usepackage{amsthm}
\usepackage{amssymb}
\usepackage{natbib}
\usepackage{graphicx}
\usepackage{url}
\usepackage{caption}
\usepackage[algo2e,ruled]{algorithm2e}

\usepackage{hyperref,setspace,natbib}

\newtheorem{theorem}{Theorem}
\newtheorem{lemma}{Lemma}
\newtheorem{proposition}{Proposition}
\newtheorem{corollary}{Corollary}
\newtheorem{remark}{Remark}

\newcommand{\hL}{\hat{L}}
\newcommand{\CX}{\mathcal{X}}

\newcommand{\goes}{\rightarrow}
\newcommand{\baS}{\bar{S}}
\providecommand{\indic}[1]{\mathbf{1}\{#1\}}
\DeclareMathOperator{\E}{E}
\DeclareMathOperator{\ERM}{ERM}
\DeclareMathOperator{\uniform}{uniform}

\usepackage{xcolor}
\usepackage{xcolor} 
\definecolor{darkblue}{rgb}{0.0, 0.0, 0.55}

\hypersetup{
    colorlinks=true,
    linkcolor=darkblue,  
    citecolor=darkblue,  
    urlcolor=darkblue,   
    filecolor=darkblue   
}

\title{Learning from Synthetic Data: Limitations of ERM\thanks{Authors ordered alphabetically. Correspondence to \texttt{\{kamin,alexbie,weiweikong,usyed,sergeiv\}@google.com}.}}
\usepackage{times}

\author{
Kareem Amin
\and 
Alex Bie
\and 
Weiwei Kong
\and 
Umar Syed
\and 
Sergei Vassilvitskii
\and \\
\Large{Google Research}
}

\date{}

\begin{document}

\maketitle

\begin{abstract}
The prevalence and low cost of LLMs have led to a rise of synthetic content. From review sites to court documents, ``natural'' content has been contaminated by data points that appear similar to natural data, but are in fact LLM-generated. In this work we revisit fundamental learning theory questions in this, now ubiquitous, setting. We model this scenario as a sequence of learning tasks where the input is a mix of natural and synthetic data, and the learning algorithms are oblivious to the origin of any individual example. 

We study the possibilities and limitations of ERM in this setting. For the problem of estimating the mean of an arbitrary  $d$-dimensional distribution, we find that while ERM converges to the true mean, it is outperformed by an algorithm that assigns non-uniform weights to examples from different generations of data. For the PAC learning setting, the disparity is even more stark. We find that ERM does not always converge to the true concept, echoing the model collapse literature. However, we show there are algorithms capable of learning the correct hypothesis for arbitrary VC classes and arbitrary amounts of contamination.





\end{abstract}
\section{Introduction}

Large language models (LLMs) represent cutting-edge advances of AI, and developers of LLMs routinely consume publicly-available datasets such as Common Crawl\footnote{\url{http://commoncrawl.org}} to improve the capabilities of their models. The first generation of LLMs were largely trained on human-generated data. However, the success of LLMs and their increased adoption has had an unexpected consequence of AI-generated content appearing in places where there was previously none. 
Thus machine learning practitioners should be aware that there is an increased chance that their training data is contaminated by LLM-generated content. 

Previous work has looked into the value of synthetic (\emph{i.e.}, AI-generated) data, and showed that while naively adding this data to the training mix may lead to model collapse, being more diligent about which data is added, the amount of curation it undergoes, and the specifics of the training process may mitigate that risk, or reverse it, leading to improved performance. These works almost uniquely focus on the LLM setting, trying to improve state of the art performance on a set of benchmarks. 

In contrast, in this work we take a traditional learning theory view on this problem. We begin by formalizing the setting and developing a framework that captures the invariants of having natural training data contaminated by synthetic additions. Specifically, we see three salient points:

\begin{itemize}\itemsep=0in
    \item {\bf Groundtruth.} There exists a (potentially small) set of natural data, coming from the true data generation distribution. 
    \item {\bf Contamination.} Natural data is repeatedly supplemented with synthetic data, which attempts to mimic or imitate the groundtruth. Moreover, 
    one cannot be certain about the origin of any individual example, making synthetic data difficult to filter out. 
    \item {\bf Repetition.} The process is continuous, the training set grows over time, adding data contaminated with the latest model in each iteration. 
\end{itemize}


Faced with this contaminated data, our goal is to find a learning algorithm which outputs models that continuously improve generalization error in every iteration. 

\subsection{Our Contributions}

\paragraph{Parameterized Contamination.} 
Learning in the presence of synthetic data has been studied theoretically by multiple authors (see Section~\ref{sec:related}). Many of these works consider an extreme, \emph{purely recursive}, setting. Aside from an initial seed set of uncontaminated data, the learner is exposed to  synthetic data in each subsequent iteration of learning that has been generated by the previous iteration's model. We extend and generalize these works by introducing a parameter $\alpha$ characterizing the degree of model contamination in each iteration, where $\alpha = 1$ corresponds to the pure recursive setting. The challenge in our setting will be that natural data and synthetic data are indistinguishable. 

\paragraph{Mean estimation.} Within this framework, we begin by studying the most fundamental statistical problem of estimating the mean of an unknown distribution. The sample average is well-known to be the empirical risk minimizer for $l_2$-error. In the synthetic data and model collapse literature this approach, gathering a dataset across multiple generations of models and weighing its examples uniformly, is well-studied and is sometimes known as an \emph{augmentation workflow}~\citep{dey2024universality} or \emph{uniform weighting}. 

Our first result (Theorem~\ref{thm:main_var_id}) is an \emph{exact} characterization of the variance of uniform weighting. The finding generalizes existing results along two key dimensions: the exact characterization holds for all contamination parameters $\alpha$, and makes almost no distributional assumptions. Our second key result (Theorem~\ref{thm:not_mvue}) is that, for all distributions in our setting, there exists a non-trivial $(\alpha < 1)$ contamination parameter, where uniform weighting is \emph{not} the minimum-variance unbiased estimate (MVUE) of the mean. 

\paragraph{PAC learning.} The mean estimation problem provides a glimpse into the kinds of problems synthetic data can introduce even in simple settings. To better understand the implications of synthetic data, we study this question in the PAC learning setting \citep{valiant1984theory}. We give a lower bound that demonstrates that the natural algorithm, repeatedly computing the ERM hypothesis on the previous round's data, does not continually improve the generalization error of the classifier (Theorem \ref{thm:lower}). It is notable that the construction holds for a very simple problem: learning 1-$d$ thresholds in the realizable setting, but the lower bound extends to VC classes of arbitrary dimensions. 

Complementing the lower bound we construct two universal algorithms, for problems with arbitrary VC dimension, that, ignoring computational constraints, produce classifiers with vanishing generalization error (Theorems \ref{thm:vc}, \ref{thm:vc_known_alpha}) regardless of the amount of data contamination.  

\subsection{Related work}\label{sec:related}
 Beginning with the paper by ~\citet{shumailov2024ai} that introduced the notion of ``model collapse,'' researchers have shown how to avoid it by being careful about how and where to use synthetic data, and have provided some theoretical justifications for its success. Much of this work concerns studying the effect of synthetic data in the context of generative models and LLMs. See, for instance, the work by ~\cite{alemohammad2024selfconsuming,Hataya_2023_ICCV,gerstgrasser2024is} on model collapse, and ~\cite{amin2025, bertrand2024on,seddik2024how, feng2024beyond,ferbach2024selfconsuminggenerativemodelscurated, firdoussi2025maximizing} on ways to combine synthetic and natural data. Our work departs from these by studying the effect of contaminated synthetic data on two fundamental learning settings: \emph{mean estimation}, and \emph{PAC learning}. 
 
 \paragraph{Mean estimation.} 
  Inspired by the model collapse literature, several authors~\citep{suresh2024rate, dey2024universality, barzilai2025models,kanabar}, have considered theoretical learning problems relevant to mean estimation in settings with model-contaminated data. \citet{dey2024universality} provide a useful taxonomy, distinguishing between the \emph{augmentation workflow} where examples are combined uniformly across learning iterations, and the \emph{discard workflow} where only the last iteration's examples are used for learning. To the best of our knowledge, our work is the first to exactly characterize the variance of mean estimation for a $d$-dimensional distribution in the augmentation workflow, with minimal distributional assumptions and arbitrary model contamination.
  
  Most of the relveant work~\citep{suresh2024rate, dey2024universality, barzilai2025models} focuses on the purely recursive setting where the data produced on each iteration comes from the previous generation's model ($\alpha = 1$, in the language of this work). The work of ~\citet{kanabar} stands apart in considering general forms of data contamination ($\alpha \not= 1$). However, they consider only discrete distributions. 
  
  Within the purely-recursive literature, we can compare our results. \citet{suresh2024rate} study the discard workflow, but for specific distributions: discrete distributions, 1-$d$ Gaussians, and 1-$d$ mixtures of Gaussians. \citet{kazdan2024collapse} characterize the variance of 1-$d$ Gaussian mean estimation in the augmentation workflow. \citet{barzilai2025models} study the augmentation workflow in a setting that subsumes Gaussian mean estimation. However, their main result requires that sample sizes grow with the number of iterations of learning. In the special case of $\alpha = 1$, we recover the $\pi^2/6$ variance noted by \citet{dey2024universality}, who analyze Gaussian mean estimation. While~\citet{dey2024universality} note the benefits of the augmentation workflow over the discard workflow, our work goes a step further and establishes the sub-optimality of even the augmentation workflow.

  \paragraph{PAC learning.} To the best of our knowledge, we are the first to consider PAC learning in the context of learning from synthetic data. \citet{valiant1985learning} and \citet{kearns1988learning} modified the PAC learning setting to include an adversary who arbitrarily perturbs a fraction of the training data, while \citet{angluin1988learning} introduced a random probability that the true label for each training example is flipped. Unlike those noise models --- which can be described, respectively, as malicious and oblivious --- in our setting the source of the noise is the learner itself, via the models it outputs in previous iterations. Finally, our main results in the PAC setting make use of results of \citet{angluin1988learning, liu2002partially} and \cite{mansouri2025learning} via reductions.

\section{Mean Estimation}

\label{sec:gaussian_mean_est}
We begin with one of the simplest learning question: given a set of samples from a distribution, estimate the mean of the distribution. In the traditional setting, it is well known that among all unbiased estimates, the empirical mean of the samples minimizes the squared loss -- the $\ell_2^2$ norm between the estimate and the truth. For many distributions (e.g. Gaussian) it is also the minimum variance unbiased estimator (MVUE) of the mean. Our first results exactly characterize the variance of the empirical mean when data is contaminated with synthetic data, and demonstrate that there is no distribution for which it is the MVUE in general.  

\subsection{Setup}
Suppose that the natural data is drawn from a $d$-dimensional distribution with mean $\mu$. At each time $t = 1,2,\ldots$, we will produce an estimate $Y_t$ for $\mu$. We are interested in studying a setting where this estimate contaminates the data produced in the subsequent round. 

A simple version of this is as follows. Suppose that on round $t$ we are given access to a pool of examples, where a $(1-\alpha)$ fraction comes from a distribution $D_0$ with mean $\mu$ and an $\alpha$ fraction comes from a distribution $D_1$ with mean $Y_{t-1}$. Let $X_t$ denote the average of these examples, which has the same distribution as $\alpha Y_{t-1} + (1-\alpha) \mu + U$, where $U \sim D(\Sigma)$ and $D(\Sigma)$ has mean $0$ and covariance $\Sigma$ (that depends on $\alpha, D_0, D_1$ and the size of the pool). 

Traditionally, one would estimate $\mu$ by taking $Y_t$ to be the average of all examples observed up to time $t$. Letting $U_t \sim D(\Sigma)$, this results in the following stochastic process:
\begin{equation}
\begin{gathered}X_1 = \mu + U_t,\quad X_{t}=\alpha Y_{t-1}+(1-\alpha)\mu+U_{t} \qquad (t\geq2),\\
Y_{t}=\frac{1}{t}\sum_{s=1}^{t}X_{s} \qquad (t\geq1).
\end{gathered}
\label{eq:stoch_proc}
\end{equation}
More generally, a learner may combine observations $\{X_t\}$ non-uniformly across rounds to form estimates $\{Y_t\}$. That is, for
\[
w=(w^{1},w^{2},\dots)\text{ for }w^{s}\in\mathbb{R}^{s},
\]
and $X_t$ is as in \eqref{eq:stoch_proc}, the learner could form the sequence
\begin{equation}
\begin{gathered}Y_{t}(w)=\sum_{s=1}^{t} w^t_s X_{s}(w) \qquad (t\geq1).
\end{gathered}
\label{eq:gen_stoch_proc}
\end{equation}
It can be shown, inductively, that for any weighting strategy in the probability simplex, the set $\{Y_t(w)\}$ consists of unbiased estimators. 

\begin{remark}
If for every $t$, the vector $w^t$ is in the probability simplex, then $\mathbb{E}[X_{t}(w)]=\mathbb{E}[Y_{t}(w^{t})]=\mu$ for any
$\alpha>0$ and $t\geq1$.
\end{remark}
Moreover, we are interested in minimizing ${\rm Var}(Y_t(w))$, since:
\begin{align*}
\mathbb{E}\left[\|Y_{t}(w)-\mu\|_2^{2}\right] &=\mathbb{E}\left[{\rm tr}\left(\left[Y_{t}(w)-\mu\right]\left[Y_{t}(w)-\mu\right]^{*}\right)\right] \\
 &={\rm tr}\mathbb{E}\left[\left[Y_{t}(w)-\mu\right]\left[Y_{t}(w)-\mu\right]^{*}\right] 
 ={\rm tr}\ {\rm Var}\left[Y_{t}(w)\right].
 \end{align*}

\subsection{Uniform Weights}

Our first result characterizes the variance of $Y_t$ in equation~(\ref{eq:stoch_proc}) with equality. The proof is straightforward; we express the variance of $Y_t$ in terms of the variance of $Y_{t-1}$, establishing the recursion ${\rm Var}(Y_t) = \left(\frac{t-1+\alpha}{t}\right)^{2} {\rm Var}(Y_{t-1}) + \frac{1}{t^{2}}\Sigma$. We then solve this recursion in terms of the gamma function. The full proof is provided in Appendix~\ref{sec:gaussian_appendix}.
\
\begin{theorem}
\label{thm:main_var_id}
It holds that
\begin{equation*}
{\rm Var}(Y_{t})=\left\{ \frac{1}{t^{2}}+\left[\frac{\Gamma(t+\alpha)}{\Gamma(t+1)}\right]^{2}\sum_{k=1}^{t-1}\left[\frac{\Gamma(k+1)}{k\cdot\Gamma(k+\alpha)}\right]^{2}\right\} \Sigma.
\end{equation*}
\end{theorem}
Before analyzing this expression further, we can already specialize Theorem~\ref{thm:main_var_id} to the edge cases $\alpha \in \{0,1\}$ using the elementary recursion $\Gamma(z+1)=z\Gamma(z)$ for $z>0$. When $\alpha = 0$, we recover that the variance of the sample mean on i.i.d. data decays like $O(1/t)$. When $\alpha = 1$ we instantiate a pure recursive learning setting, where the learner has access to natural data only on one round, and subsequently re-trains on a dataset augmented by its own predictions. This results in a far worse estimator, which has constant error regardless of how many rounds elapse, recovering the $\pi^2/6$ constant observed by ~\citet{dey2024universality}. 

\begin{remark}
We have the following edge cases. For $\alpha=0$,
\begin{align*}
{\rm Var}(Y_{t}) & =\left(\frac{1}{t^{2}}+\left[\frac{\Gamma(t)}{\Gamma(t+1)}\right]^{2}\sum_{k=1}^{t-1}\left[\frac{\Gamma(k+1)}{k\cdot\Gamma(k)}\right]^{2}\right)\Sigma =\frac{1}{t}\Sigma.
\end{align*}
For $\alpha=1$, 
\begin{align*}
{\rm Var}(Y_{t}) & =\left(\frac{1}{t^{2}}+\left[\frac{\Gamma(t+1)}{\Gamma(t+1)}\right]^{2}\sum_{k=1}^{t-1}\left[\frac{\Gamma(k+1)}{k\cdot\Gamma(k+1)}\right]^{2}\right)\Sigma =\left(\sum_{k=1}^{t}\frac{1}{k^{2}}\right)\Sigma\overset{t\to\infty}{\to}\frac{\pi^{2}}{6}\Sigma.
\end{align*}
\label{rem:edge_cases}
\end{remark}

We can  also use Gautschi's inequality to further analyze the expression for ${\rm Var}(Y_t)$. Remarkably, we can characterize the covariance of this vector up to constant factors. The full proof is in Appendix~\ref{sec:gaussian_appendix}. 

\begin{lemma}
\label{lem:gautschi}(Gautschi's inequality) For $z>0$ and $\lambda\in(0,1)$,
it holds that
\[
z^{1-\lambda}\leq\frac{\Gamma(z+1)}{\Gamma(z+\lambda)}\leq(z+1)^{1-\lambda}.
\]
\end{lemma}

\begin{theorem}
For $t\geq3$, we have that
\begin{equation}
\frac{1}{2}\left[\frac{1}{t}+\frac{1}{t^{2}}+\frac{1}{t^{2(1-\alpha)}}\right]\Sigma\preceq{\rm Var}(Y_{t})\preceq4\left[\frac{1}{t}+\frac{1}{t^{2}}+\frac{1}{t^{2(1-\alpha)}}\right]\Sigma.\label{eq:main_var_bounds}
\end{equation}
where $\preceq$ indicates Loewner ordering. 
\end{theorem}

Theorem~\ref{eq:main_var_bounds} allows us to make multiple observations. First, when $\alpha \leq 1/2$, i.e. the data contains more natural examples than synthetic examples, the variance is $\Theta(1/t)$. At $\alpha = 1/2$ there is a phase change, and the dominant term becomes $\Theta(t^{2(1-\alpha)})$ when $\alpha > 1/2$. 

We also observe that for any distribution $D(\Sigma)$ where the sample average is the MVUE for $\mu$ in the standard i.i.d. setting (such as the Gaussian distribution), there is no $\alpha$ and alternative weighting strategy $w$ that can enjoy a rate that is better than $\Theta(1/t)$ in $t$. The proof follows by contradiction. Suppose that a $o(1/t)$ strategy existed. Then, it can be used to attain the same rate for $\alpha = 0$ (i.e. the standard i.i.d. setting). This follows by receiving i.i.d. samples $\mu + U_t$, for each round $t$, and simulating the recursion in \eqref{eq:gen_stoch_proc} with the weighting strategy $w$. This, in turn, contradicts the fact that uniform weighting is the MVUE in the standard i.i.d. setting.

\begin{theorem}
Suppose the sample average is the MVUE for estimating $\mu$ from i.i.d. data $\{X_t \sim \mu + U_t\}$, $U_t \sim D(\Sigma)$. Then there is no $\alpha$, weighting scheme $w$, and function $f_\alpha(t)$ such that
${\rm Var}(Y_t(w)) \preceq f_{\alpha}(t)\Sigma$ and $f_\alpha(t) = o(1/t)$.
\end{theorem}

Although uniform weighting sometimes achieves the optimal dependence on $t$ when $\alpha \leq 1/2$, we can also show that uniform weighting is \emph{not} the MVUE in general; in particular, it is not the MVUE for large $\alpha$. We remark that when $\alpha = 1$, there is a separation between the uniform weighting strategy and the simple strategy that assigns, for every $t$, $w^t_{{\rm simple}} = (1, 0, 0, \dots, 0)$, which follows from Remark~\ref{rem:edge_cases}. 

\begin{remark}
For $\alpha = 1$, and sufficiently large $t$, ${\rm Var}(Y_t(w_{\rm simple})) = \Sigma \prec {\rm Var}(Y_t(w_{\rm uniform}))$.
\end{remark}

The next theorem investigates whether the sub-optimality of uniform weighting is a phenomenon that happens only in the degenerate case when $\alpha = 1$. We show that this is false. There are non-trivial $\alpha \in (0,1)$ for which uniform weighting is also not the MVUE. For the following theorem, let $Y_t^{\alpha}(w)$ indicate the estimator under weighting scheme $w$ and data-contamination parameter $\alpha$. 

\begin{theorem}\label{thm:not_mvue}
There exists a family of weighting schemes $\{w_\alpha\}_{\alpha \in [0,1]}$ such that there exists an $\alpha^* < 1$ where ${\rm Var}(Y_t^{\alpha}(w_\alpha)) \prec {\rm Var}(Y_t^{\alpha}(w_{{\rm uniform}}))$ for every $\alpha \in (\alpha^*, 1]$. 
\end{theorem}
The proof follows by observing that uniform weighting is not optimal for $\alpha = 1$. We then consider the family of strategies $\hat{w}_\alpha$ that upweight data from the first round, by taking $w^1_{\alpha} = 1$ and:

\begin{align*}
\hat{w}_{\alpha}^{s+1} & =\left(\tfrac{1}{1+(s-1)(1-\alpha)},\tfrac{1-\alpha}{1+(s-1)(1-\alpha)},\ldots\tfrac{1-\alpha}{1+(s-1)(1-\alpha)}\right).
\end{align*}
We can show that the variance of $Y_t^{\alpha}(\hat{w}_{\alpha})$ is a continuous function in $\alpha$ from which we can conclude that there are non-trivial $\alpha \in (0, 1)$ such that $\hat{w}_\alpha$ has lower variance than uniform weighting. The full proof can be found in Appendix~\ref{sec:improved_weighting}.

\section{PAC Learning}
Our mean estimation results demonstrate that ERM is not an optimal solution for even a simple distributional estimation problem. We now study the same question from the perspective of classification in the PAC setting. We show that for more complex settings, synthetic data can wreak havoc on traditional learning approaches, even when it's mixed with natural data at each iteration. 

One of the most powerful theorems in classical learning theory is that ERM guarantees a generalization rate for concept classes with finite VC dimension. Moreover, the excess error of the ERM classifier over the true optimal classifier in a hypothesis class is vanishing with more data. Our first main theorem in this section will demonstrate that, in general, this is not true for settings where data labeled according to previous hypotheses are mixed in with data labeled according to nature. For a very simple concept class, linear separators, repeatedly finding the empirical risk minimizer can result in generalization error that stalls. Mirroring the mean estimation setting, this occurs once the contamination rate of the data surpasses the critical point $\alpha > \frac{1}{2}$.

Subsequently, we will show that there are other algorithms (not repeated ERM) that achieve vanishing generalization error for concept classes with finite VC dimension regardless of the data-contamination rate. 

\subsection{Setup}
Let $\CX$ be the domain of examples, and consider the label space $\{0,1\}$. Let $F \subseteq \{0,1\}^{\mathcal X}$ be a hypothesis class. 
We are interested in learning algorithms that are called as part of an iterative training loop, and whose output can depend on the hypotheses that were output in previous iterations.

Accordingly, we define a learning algorithm to be a (possibly randomized) function $A : 2^{\mathcal X \times \{0,1\}} \times  (\{0,1\}^{\mathcal X})^* \to  \{0,1\}^{\mathcal X}$. In other words, a learning algorithm $A$ takes as input a dataset $S \subseteq \mathcal X \times \{0,1\}$ and a sequence of hypotheses $f_0, \ldots, f_{t-1} \in F$, and outputs a (possibly randomized) hypothesis $f_t = A(S, f_0, \ldots, f_{t-1})$.

Let $D$ be a distribution on $\CX$. In this paper, we consider the \emph{realizable} case, where some $f^* \in F$ is the true concept that labels examples. The \emph{loss} of a hypothesis $f \in \{0,1\}^{\mathcal X}$ is defined
\[
L(f) := \Pr_{x \sim D} [f(x) \not = f^*(x)].
\]

For any dataset $S \subseteq \CX \times \{0,1\}$ the \emph{empirical loss} of hypothesis $f \in F$ on $S$ is defined
\[
\hL(f, S) = \Pr_{(x,y) \sim \mathrm{Unif}(S)} [f(x) \not = y].
\]
A common learning algorithm is \emph{empirical risk minimization}, denoted $A_{\textrm{ERM}}$, and defined
\[
A_{\ERM}(S, f_0, \ldots, f_{t-1}) = \arg \min_{f \in F} \hL(f, S).
\]
Note that ERM ignores the previous hypotheses and depends only on the dataset.

\begin{algorithm2e}[!ht]
\caption{Recursive learning \label{alg:recursive}}
\DontPrintSemicolon
\KwIn{Sample space $\CX$, label space $\{0,1\}$, distribution $D$ on $\CX$, function class $F \subseteq \{0,1\}^{\mathcal X}$, true concept $f^* \in F$, sample size $n$, synthetic data rate $\alpha \in [0, 1]$, learning algorithm $A$.}
$\baS_0 \gets \{(x^{0}_i, y^{0}_i) : i \in [n], x^{0}_i \sim D, y_i \sim f^*(x_i)\}$\;
$f_0 \gets A(\baS_0)$ \;
\For{$t = 1, 2, 3, \ldots$} {
\For{$i = 1,\ldots,n$} {
$x^{t}_i \sim D$\;
$b^t_i \sim \textrm{Bernoulli}(\alpha)$\;
$y_{\rm model} \sim f_{t-1}(x^{t}_i), y_{\rm true} \sim f^*(x^{t}_i)$\;
$y^{t}_i = b^t_i y_{\rm model} + (1-b^t_i) y_{\rm true}$
}
$S_t \gets \{(x^{t}_i, y^{t}_i) : i \in [n]\}$\;
$\baS_t \gets \baS_{t-1} \cup S_t$\;
$f_t \gets A(\baS_t, f_0, \ldots, f_{t-1})$\;
}
\end{algorithm2e}

Algorithm \ref{alg:recursive} formalizes our learning problem. Learning algorithm $A$ is used to estimate an initial hypothesis $f_0$ from dataset $\baS_0$, where each sample in $\baS_0$ is drawn from $D$, and its label is given by the true concept $f^*$. In each subsequent round $t \ge 1$, a dataset $S_t$ is added to $\baS_{t-1}$ to form $\baS_t$, where again each sample in $S_t$ is drawn from $D$, but its label is given by $f_{t-1}$ with probability $\alpha \in [0,1]$ and otherwise given by $f^*$. Learning algorithm $A$ is then used to estimate the next hypothesis $f_t$ from $\baS_t$ and $f_0, \ldots, f_{t-1}$.

\paragraph{Problem Statement.}
Our {\bf goal} is to specify a learning algorithm $A$ such that $\lim_{t \goes \infty} \E[L(f_t)] = 0$, where the expectation is with respect to the randomness in Algorithm \ref{alg:recursive} (which includes the randomness in $A$).

\paragraph{Remark on set notation}

We use the variable $S$ to define sets that we actually want to think of as sequences, \emph{i.e.}, they preserve multiplicity and order. We denote them as sets because doing so affords us the convenience of using the element operator $\in$, the union operator $\cup$, and set builder notation.




\subsection{Repeated ERM and Threshold Functions}

Our first theorem shows that repeated ERM does not attain our goal in general. While the proof is somewhat technical, its key ideas are straight-forward. Consider a very simple concept class: learning one dimensional thresholds on the interval. Define $D$ according to the following p.m.f:

\[ 
p(x) = \begin{cases} \frac{1}{2} - \frac{1}{2n} &\text{if $x = +1$}\\ \frac{1}{n}  &\text{if $x = 0$}\\ \frac{1}{2} - \frac{1}{2n}  &\text{if $x = -1$}\end{cases}
\]

There is a constant probability $(1-\frac{1}{n})^n$ that the learner does not see the example $x = 0$ in the $0$-th round of learning. Suppose this happens. In the absence of any information on how this example should be labeled, the learner mislabels it with constant probability. 

Next, as long as $\bar{S}_t$ contains more mis-labeled than correctly labeled examples for $x=0$, ERM will continue to mislabel it. Thus, the only way for the learner to recover is to see more naturally-labeled examples than model-labeled examples for the example $x=0$. 

The final element of the proof is to consider a process $\{\Delta_k\}$ which increments whenever the learner encounters a model-provided label and decrements whenever the learner encounters a nature-provided label for $x = 0$. This process can be understood as a biased random walk on the integers. Thus, the probability that the learner never recovers can be directly related to the probability that the biased random walk is transient. Finally, the one-dimensional problem can be embedded in a linear separator of arbitrary dimension, resulting in Theorem~\ref{thm:lower}. The full proof is provided in Appendix~\ref{sec:hardness}. 

Let $\{W_t\}$ denote a random walk with bias parameter $\alpha$. Let $W_0 = 0$, $W_{t+1} = W_t + 1$ with probability $\alpha$ and $W_{t+1} = W_t - 1$ with probability $(1-\alpha)$. 

\begin{theorem}
\label{thm:lower}
 Define $c^*_\alpha = \Pr[\forall t \geq 1, W_t \geq 1]$. For any sample size $n \geq 2$ and $d \geq 2$, there exists a distribution $D$, a hypothesis class $F$ with VC dimension $d$, and $f^* \in F$, such that if $A = A_{\ERM}$ then
\[
\liminf_{t \goes \infty} \E[L(f_t)] \ge \frac{c^*_\alpha}{8n}.
\]\end{theorem}

The following Lemma establishes that for $\alpha > 1/2$, $c^*_\alpha > 0$. It is a fairly standard result in probability that a biased random walk is not recurrent. However, we provide the full proof in the Appendix for completeness. 

\begin{lemma}
When $\alpha > 1/2$, the event that the random walk $\{W_t\}$ stays positive forever occurs with non-zero probability: $\mathbb{P}[\forall t\geq 1, W_t \geq 1] > 0$
\end{lemma}

\begin{corollary}
\label{corr:lower}
For any sample size $n \geq 2$ and $d \geq 2$, there exists a distribution $D$, a hypothesis class $F$ with VC dimension $d$, and $f^* \in F$, and $\alpha > 1/2$ such that if $A = A_{\ERM}$ then
\[
\liminf_{t \goes \infty} \E[L(f_t)] > 0 
\]\end{corollary}

\subsection{Simple Universal Algorithm}

The next theorem establishes that there exists a learning algorithm that achieves our goal for any VC class, and any $\alpha$. The algorithm requires access to $f_{\rm uniform}$: a classifier that ignores its input and returns a random label. $f_{\rm uniform}$ is used to occasionally collect labels that are distributed like $f^*(x)$ with random classification noise added. On other rounds, we rely on classical learning theory results~\citep{angluin1988learning} to learn with label noise. We show that by tuning the probability of playing the uniform classifier, vanishing generalization error can be achieved. 

The algorithm has a number of drawbacks. Its use of the classifier $f_{\rm uniform}$ means that the algorithm occasionally deploys a classifier with maximum loss. While the probability of deploying the random classifier vanishes, this is nevertheless unrealistic and unnatural. More importantly, its generalization error vanishes at the relatively slow rate of $O(t^{-1/4})$. 

Nevertheless, we include this result as it serves as a simple witness, with an easy proof, that moving away from repeated ERM allows one to learn in this setting. Our next result requires a more complicated algorithm, but achieves a more standard $O(t^{-1/2})$ rate. 

\begin{theorem} \label{thm:vc} Let $F \subseteq \{0,1\}^{\mathcal X}$ be a hypothesis class with VC dimension $d$. There exists a learning algorithm $A$ outputting classifiers $f_t$ such that for all $t \ge 0$
\[
\E[L(f_t)] \le O\left(\tfrac{\sqrt{d \log (nt)}}{(1-\alpha)(nt)^{1/4}}\right) + \tfrac{1}{\sqrt{n(t + 1)}} + \exp(-\Omega(\sqrt{t/n})).
\]\end{theorem}
\begin{proof} For any $t \ge 0$, given dataset $\baS_t$ and hypotheses $f_0, \ldots, f_{t-1}$ as input, we define the output $f_t = A(\baS_t, f_0, \ldots, f_{t-1})$ of learning algorithm $A$ as follows. With probability $\frac{1}{\sqrt{n(t+1)}}$ let $f_t = f_{\uniform}$. Otherwise, let
\[
\baS'_t = \left\{(x_i, y_i) \in S_r : f_{r-1} = f_{\uniform} \textrm{ for } r \in [t]\right\}
\]
and $f_t = A_{\ERM}(\baS'_t, f_0, \ldots, f_{t-1})$. In other words, with probability $1 - \frac{1}{\sqrt{n(t + 1)}}$, let $f_t$ be the empirical risk minimizer for the dataset consisting of labeled samples drawn in previous rounds where the labels were generated by either $f_{\uniform}$ or $f^*$. Let $U_t = \sum_{r=1}^t \indic{f_{r-1} = f_{\uniform}}$ be the number of these rounds, a random variable. Observe that $\baS'_t$ contains $nU_t$ samples, and each $(x, y) \in \baS'_t$ is independent and distributed as follows: Draw $x$ from $D$ and let $y = f^*(x)$ with probability $(1-\alpha) + \frac \alpha  2 = \frac12 + \frac{1-\alpha}{2}$, and otherwise let $y = 1 - f^*(x)$. By Theorem 2 of \cite{angluin1988learning} we have
$$\E[L(f_t) ~|~ f_t \neq f_{\uniform} \wedge nU_t \ge m] \le O\left(\tfrac1{1-\alpha}\sqrt{\tfrac{d \log m}{m}}\right)$$
Note that $\E[U_t] = \sum_{r=1}^t \frac{1}{\sqrt{n r}} = \Theta(\sqrt{t/n})$. Also, since each $\indic{f_t = f_{\uniform}}$ is an independent Boolean random variable, by the multiplicative Chernoff bound we have $\Pr[U_t < O(\E[U_t])] \le \exp(-\Omega(\E[U_t]))$. Therefore letting $m = \Theta(\sqrt{nt})$ we have

\begin{align*}
\E[R(f_t)] &\le \E[L(f_t)]\\
&\le \E[L(f_t) ~|~ f_t \neq f_{\uniform} \wedge nU_t \ge m] + \Pr[f_t = f_{\uniform} \vee nU_t < m]\\
&\le \E[L(f_t) ~|~ f_t \neq f_{\uniform} \wedge nU_t \ge \Omega(\sqrt{nt})] + \Pr[f_t = f_{\uniform} \vee nU_t < O(\sqrt{nt})]\\
&\le \E[L(f_t) ~|~ f_t \neq f_{\uniform} \wedge nU_t \ge \Omega(\sqrt{nt})] + \Pr[f_t = f_{\uniform}] + \Pr[nU_t < O(\sqrt{nt})]\\
&\le \E[L(f_t) ~|~ f_t \neq f_{\uniform} \wedge nU_t \ge \Omega(\sqrt{nt})] + \Pr[f_t = f_{\uniform}] + \Pr[U_t < O(\E[U_t])]\\
&\le O\left(\tfrac{\sqrt{d \log (nt)}}{(1-\alpha)(nt)^{1/4}}\right) + \tfrac{1}{\sqrt{n(t + 1)}} + \exp(-\Omega(\E[U_t]))\\
&\le O\left(\tfrac{\sqrt{d \log (nt)}}{(1-\alpha)(nt)^{1/4}}\right) + \tfrac{1}{\sqrt{n(t + 1)}} + \exp(-\Omega(\sqrt{t/n})).
\end{align*}
\end{proof}

\subsection{Learning Disagreements from Positive Examples}

In this section, we present an algorithm that achieves vanishing $O(t^{-1/2})$ generalization error for any contamination-rate $\alpha$. Unlike the prior simple algorithm which only learns from examples in select, designated rounds, this section's algorithm learns from samples collected in all rounds. It does so by making better use of the fact that synthetic examples are labeled by a known classifier output in the previous round.

\begin{theorem}\label{thm:vc_known_alpha}
Let $F \subseteq \{0,1\}^{\mathcal X}$ be a hypothesis class with VC dimension $d$. There exists a learning algorithm A outputing classifiers $f_t$ such that for all $t \geq 0$
\[
\E[R(f_t)] \le \tilde O\left(\sqrt{\tfrac {d} {(1-\alpha) nt} }\right).
\]
\end{theorem}

The analysis proceeds via a reduction to \emph{learning with positive and unlabeled examples} (also referred to as \emph{PU learning}). We make use of a result from the seminal work of \cite{liu2002partially}, as re-stated as Corollary 6 in \cite{mansouri2025learning}.

\begin{lemma}[PU learner \citep{liu2002partially}; Corollary 6 from \cite{mansouri2025learning}]\label{lem:pu_learner}
Let $F \subseteq \{0,1\}^{\mathcal X}$ be a hypothesis class with VC dimension $d$. For a distribution $D$ over $\mathcal X$ and true concept $f^* \in F$, denote by $D^+$ the distribution $D(x | f^*(x)=1)$. There exists an algorithm such that when given $m_P(\varepsilon, \delta) = O(\tfrac {d\log(1/\varepsilon) + \log(1/\delta)} \varepsilon)$ positively labeled samples from $D^+$ and $m_U(\varepsilon, \delta) = O(\tfrac {d\log(1/\varepsilon) + \log(1/\delta)} \varepsilon)$ unlabeled samples from $D$, it outputs a hypothesis $f$ such that $L(f) \leq \varepsilon$ with probability $\geq 1 -\delta$.
\end{lemma}

We discuss the connection between the settings, which conveys the main idea of the proof. In PU learning, the learner receives a random sample of labeled examples from the positive class, and a random sample of unlabeled examples from both classes. In round $t$ of recursive learning, examples received by the learner are labeled by either the true concept $f^*$ or the previous model $f_{t-1}$ (which the learner has access to); for each individual example received, the learner does not know which rule it was labeled by.

Now, consider the task of learning the disagreement between the previous model and the true concept $\{x \in \mathcal X:  f_{t-1}(x) \neq f^*(x)\}$. Since we know the $f_{t-1}$ \emph{exactly}, successfully learning the disagreement means we have learned the true concept $f^*$. Whenever we observe a disagreement between an example's label and the previous model's prediction on the point, we can be sure the label comes from the true concept -- this corresponds to a positive example for the task of learning the disagreement. When the example's label and the previous model's prediction agree, we cannot distinguish between: (a) the previous model and true concept agree on the point (it is a negative example for the task of learning the disagreement); or (b) the previous model and true concept disagree on the point, but we fell into the $\alpha$ probability case where the example was labeled with the previous model. Hence the example is considered unlabeled for the task of learning the disagreement. 

The above discussion suggests a natural algorithm for recursive learning: learn the disagreement between the previous round's model and the true concept via PU learning, and use the disagreement to form the next round's hypothesis. In the following, we formalize this idea.

\paragraph{Algorithm.}

The algorithm attaining the bound of Theorem \ref{thm:vc_known_alpha} proceeds in collections of rounds we call \emph{epochs}. Epochs are indexed by $k\geq 1$. For all the rounds $t$ in the same epoch $k$, the algorithm will output the same classifier which we denote as $g_k$ $\in F$. Note that each epoch may comprise of a different number of rounds. Formally, let $\mathsf k(t): \mathbb Z^+ \to \mathbb Z^+$ denote the mapping of a round $t$ to its corresponding epoch $k$. $\mathsf k(t)$ is a monotonically increasing step function, with $\mathsf k(1) = 1$. Our algorithm outputs $f_t = g_{\mathsf k(t)}$, with $g_1 = f_0$ as defined in Algorithm \ref{alg:recursive}.

\emph{Update rule:} After the designated number of rounds in epoch $k$ using classifier $g_k$, we use an update rule to produce $g_{k+1}$, as a function of $g_k$ and data points collected in the epoch $T_k = \bigcup_{\{t: \mathsf k (t) = k\}} S_t,$ where $S_t$ are the examples collected in round $t$. 

The full update rule is given in Algorithm \ref{alg:update}. It proceeds by mapping the learning problem in $F$ into a learning problem in the XOR class which we denote as $g_k \oplus F$. For bits $\in \{0,1\}$ we denote by $\oplus$ the XOR operation. For two binary functions $f, g : \mathcal X \to \{0,1\}$, we define $(f \oplus g) (x) := f(x) \oplus g(x)$ for all $x \in \mathcal X$. For a class $F$ and a binary function $g$, $g \oplus F := \{g \oplus f : f \in F\}$. The learning problem in the XOR class is a PU learning problem, whose solution can be converted to a solution to the original problem.

\begin{algorithm2e}[!t]
\caption{Epoch hypothesis update rule \label{alg:update}}
\DontPrintSemicolon
\SetKw{KwAnd}{\textbf{and}}
\KwIn{Function class $F \subseteq \mathcal X \to \mathcal Y$. Positive unlabeled learner $A$ (see Lemma \ref{lem:pu_learner}), target number of positive samples $p_k$, target number of unlabeled samples $u_k$. Epoch hypothesis $g_k$, epoch samples $T_k = \{(x_i,y_i)\}_{i=1}^{m+u_k}$. }
\KwResult{Next epoch hypothesis $g_{k+1}$.}

$T_k^{\oplus} \gets \{ (x_i, y_i \oplus g_k(x_i)): (x_i, y_i) \in T_k\}$\;
$P_k^\oplus \gets \{(x_i, 1): (x_i, y_i^\oplus) \in T_k^\oplus \text, y_i^\oplus = 1, i \in [1, m]\}$\;
$U_k \gets \{x_i : i \in [m+1, m+u_k]\}$\;
\eIf{$|P_k^\oplus| \geq p_k$ \KwAnd $|U_k| \geq u_k$}
{$h \gets A(P_k, U_k; g_k \oplus F)$\;
$g_{k+1} \gets h \oplus g_k $\;}
{$g_{k+1} \gets g_k$}
\end{algorithm2e}

\emph{Schedule:} We choose the epoch schedule $\mathsf k (t)$ as well as a target error schedule $\varepsilon_k$ to achieve our desired bounds. The algorithm has the property that for all $k$ sufficiently large, $\E[L(g_k)] \leq 2\varepsilon_k$. Specifically, we take $\varepsilon_k = \tfrac{1}{2^{k-1}}$ and $\mathsf k (t)$ such that each epoch is $O(\tfrac {d\log^2(1/\varepsilon_{k+1})} {(1-\alpha) n \varepsilon_{k+1}^2})$ rounds.


\begin{lemma}\label{lem:improvement}
Suppose we are in epoch $k \geq 1$. Let $\varepsilon_{k+1}$ be the next epoch's target error, and suppose the current epoch samples $T_k$ produced in the manner described in Algorithm \ref{alg:recursive} with $g_k$ satisfies

\begin{align*}
|T_k| &\geq m_P(\varepsilon_{k+1}, \tfrac \delta 2) \cdot \tfrac  {8\log\left(\tfrac {2} {(\delta)}\right)} {(1-\alpha) \varepsilon_{k+1}} + m_U(\epsilon_{k+1}, \tfrac \delta 2) = O\left(\tfrac {d\log(1/\varepsilon)\log(1/\delta) + \log^2(1/\delta)} {(1-\alpha)\varepsilon_{k+1}^2}\right)
\end{align*}

where $m_P, m_U : (0,1)^2 \to \mathbb N$ are sample complexity functions from Lemma \ref{lem:pu_learner}. Setting $p_k = m_P(\varepsilon_{k+1}, \tfrac \delta 2) = O(\tfrac{d\log(1/\varepsilon_{k+1}) + \log(1/\delta)} {\varepsilon_{k+1}})$ and $u_k = m_U(\varepsilon_{k+1}, \tfrac \delta 2) = O(\tfrac{d\log(1/\varepsilon_{k+1}) + \log(1/\delta)} {\varepsilon_{k+1}})$ in Algorithm \ref{alg:update}, we have that with probability $\geq 1 - \delta$, Algorithm \ref{alg:update} outputs $g_{k+1}$ satisfying $L(g_{k+1}) \leq \varepsilon_{k+1}$.
\end{lemma}

\begin{proof}[Proof of Theorem \ref{thm:vc_known_alpha}]
Given Lemma \ref{lem:improvement}, we prove the result. We consider the target error schedule $\varepsilon_k = \frac 1 {2^{k-1}}.$ Note that the only hypothesis we have no control over, $g_1$, trivially satisfies this requirement. We define the epoch schedule $\mathsf k (t)$ such that each round gets at least enough samples to satisfy the hypothesis of Lemma \ref{lem:improvement}. Indeed, denote by $r_k$ the required size for $|T_k|$ in Lemma \ref{lem:improvement}.
$$r_k \leq C_1 \frac {d\log(1/\varepsilon_{k+1})\log(1/\delta) + \log^2(1/\delta)} {(1-\alpha)\varepsilon_{k+1}^2}$$
for some absolute constant $C_1$ and $k$ sufficiently large. We set $\delta = \varepsilon_{k+1}$. Under this setting of $\delta$, the high probability bound of Lemma \ref{lem:improvement} implies $\E [g_{k+1}] \leq 2\varepsilon_{k+1}$.

We choose $\mathsf k(t)$ such that for all $k \geq 1$, $|\{t: \mathsf k (t) = k\}| = \lceil r_k/n\rceil$. Since inside an epoch, we output the same hypothesis $g_k$, it suffices to check if the final round in an epoch satisfies the stated error rate in Theorem \ref{thm:vc_known_alpha}. For $k\geq1$, the last round of a given epoch $t_k := \max \{t: \mathsf k (t) = k\}$ satisfies
\begin{align*}
t_k = \sum_{j=1}^k \lceil r_j/n \rceil \leq C_2 \tfrac d {(1-\alpha) n} \sum_{j=1}^k \log^2(\tfrac 1 {\varepsilon_{j+1}}) \cdot \tfrac 1 {\varepsilon_{j+1}^2} = C_2 \tfrac d {(1-\alpha) n}  \sum_{j=1}^k (j2^j)^2 \leq C_3 \tfrac d {(1-\alpha) n} (2^kk)^2
\end{align*}
for absolute constants $C_2$, $C_3$ and $k$ sufficiently large. To conclude, we have for $k$ sufficiently large 
$$\E [L(f_{t_k})] = \E [L(g_k)] \leq 4 \cdot \tfrac 1 {2^k} \leq 4 \cdot \sqrt{\tfrac {C_3 d} {(1-\alpha) nt_k}} \cdot k = \tilde O\left(\sqrt{\tfrac d {(1-\alpha) n t_k}}\right).$$
\end{proof}

\begin{proof}[Proof of Lemma \ref{lem:improvement}]
First, we show reducing to positive unlabeled learning is valid.

\emph{The preconditions for invoking the PU learner are satisfied}. The class we are invoking the PU learner on is $g_k \oplus F$, which has the same VC dimension as $F$. Therefore the stated sample complexity suffices for learning $g_k \oplus F$. The samples $T_k^{\oplus}$ come from the original marginal distribution $D$, and are labeled $0 = (g_k \oplus f^*)(x)$ everywhere $g_k(x) = f^*(x)$. For $x$ with $g_k(x) \not= f^*(x) \iff (g_k \oplus f^*)(x) = 1$, the label is $1$ with probability $1-\alpha$ and $0$ with probability $\alpha$ independently. Hence we can conclude that $P_k^\oplus$ is indeed an i.i.d. sample from $D(x|(g_k \oplus f^*)(x) = 1)$. Also, $U_k$ is indeed an i.i.d. sample from $D$.

\emph{The output satisfies our desired error guarantee with respect to the original problem.} With probability $\geq 1 - \delta$, our PU learner for $g_k \oplus F$ outputs $h$ with $\Pr_{x \sim D} [h(x) \not = (g_k \oplus f^*)(x)] \leq \varepsilon_{k+1}$. Since for all $x$, $h(x) \not = (g_k \oplus f^*)(x) \iff (h \oplus g_k)(x) \not = f^*(x)$, we have $\Pr_{x \sim D} [(h \oplus g_k)(x) \not = f^*(x)] \leq \varepsilon_{k+1}$.

Hence, whenever the learner is invoked, our update rule succeeds with probability $\geq 1 - \tfrac \delta 2$. To finish the proof, we consider two cases:

\emph{Case 1:} $L(g_k) \leq \varepsilon_{k+1}$. If the error of the current hypothesis $g_k$ meets the target error threshold, we either: (1) gather enough samples to invoke the learner, which means our update rule succeeds with high probability; or (2) fail to gather enough and pass on $g_k$ to the next round, which by assumption, achieves the target error.

\emph{Case 2:} $L(g_k) > \varepsilon_{k+1}$. By applying a Chernoff bound and plugging in the requested size of $T_k$ and using the assumption $L(g_k) > \varepsilon_{k+1}$, we have that with probability $\geq 1- \tfrac \delta 2$, $|P_{k}^\oplus| > m_P(\varepsilon_{k+1}, \tfrac \delta 2)$, that is, we obtain enough positive examples required to invoke the PU learner. This calculation is deferred to Appendix \ref{sec:update_rule_lemma_calculation}. Taking a union bound over the failure probability of the PU learner concludes the proof.
\end{proof}

\paragraph{Discussion.}
The algorithm requires knowledge of $\alpha$, the contamination rate, since it is used to define each epoch as lasting $\tilde{O}(\tfrac {d} {(1-\alpha) n \varepsilon_{k+1}^2})$ rounds. We note that there are simple techniques for estimating $\alpha$ from data. For example, by outputting the all-zeros and all-ones classifiers across two rounds, one expects to observe $\alpha n$ disagreements in total. We do not fully explore whether knowledge of $\alpha$ can be removed, and assume that it is known exactly for simplicity.

We also observe that it is possible to get an $O(1/nt)$ rate in Theorem \ref{thm:vc_known_alpha} (which is more natural to expect in the realizable setting we study) if we make an adjustment to the setup. In Theorem \ref{thm:vc_known_alpha}, the expected accuracy requirement is for \emph{all} rounds. This choice in our modeling framework reflects the (realistic) objective of steady improvement from one model generation to the next. At a high level, Theorem \ref{thm:vc_known_alpha}'s algorithm functions by collecting feedback from errors, and hence, releasing a good model every round limits the amount of feedback received and increases the overall sample complexity. If our concern was to do well only with respect to the \emph{final} round $t$, we could release a model with constant error in rounds $1,2,...,{t-1}$, resulting in more feedback received, and the desired $O(1/nt)$ error rate for $f_t$. We leave open the question of whether it is possible to achieve the $O(1/nt)$ error rate for all rounds simultaneously.

\section{Conclusion}
The prevalence of synthetic data introduces a new class of learning questions, where the training set is contaminated by additional examples. We formally study this phenomenon, for two fundamental problems: mean estimation and PAC learning. 

For mean estimation, we give a full characterization of the variance of the most fundamental estimator, the empirical mean. We show that this is not the MVUE in general, including for distributions where the empirical mean is the MVUE in i.i.d. settings (e.g. Gaussian distributions), and when the data is not fully contaminated ($\alpha < 1$). In the PAC setting, we show that repeated ERM can experience generalization error that stalls. We complement this with two algorithms that do achieve vanishing generalization. 

Interesting open problems for mean estimation include fully characterizing the minimum variance unbiased estimator, and allowing the mean to depend on a vector of covariates instead of remaining fixed in every round. For PAC learning, open problems include expanding the results to the agnostic setting, and obtaining sample complexity lower bounds for generic algorithms.    

\newpage
\bibliographystyle{plainnat}
\bibliography{ref}

\newpage
\appendix

\section{Appendix}

\subsection{Mean Estimation Proofs}\label{sec:gaussian_appendix}

\begin{theorem}
It holds that
\begin{equation}
{\rm Var}(Y_{t})=\left\{ \frac{1}{t^{2}}+\left[\frac{\Gamma(t+\alpha)}{\Gamma(t+1)}\right]^{2}\sum_{k=1}^{t-1}\left[\frac{\Gamma(k+1)}{k\cdot\Gamma(k+\alpha)}\right]^{2}\right\} \Sigma.
\end{equation}
\end{theorem}
\begin{proof}
Let $V_{t}={\rm Var}(Y_{t}(\bar{w}))$. By definition, we have
\begin{align*}
V_{t} & ={\rm Var}\left(\frac{1}{t}\sum_{s=1}^{t}X_{s}\right)={\rm Var}\left(\frac{X_{t}}{t}+\frac{1}{t}\sum_{s=1}^{t-1}X_{s}\right)\\
 & ={\rm Var}\left(\frac{X_{t}}{t}\right)+{\rm Var}\left(\frac{1}{t}\sum_{s=1}^{t-1}X_{s}\right)+2{\rm Cov}\left(\frac{X_{t}}{t},\frac{1}{t}\sum_{s=1}^{t-1}X_{s}\right)\\
 & ={\rm Var}\left(\frac{\alpha Y_{t-1}+U_{t-1}}{t}\right)+{\rm Var}\left(\frac{t-1}{t}Y_{t-1}\right)+2{\rm Cov}\left(\frac{X_{t}}{t},\frac{t-1}{t}Y_{t-1}\right)\\
 & =\frac{\alpha^{2}V_{t-1}+\Sigma}{t^{2}}+\left(\frac{t-1}{t}\right)^{2}V_{t-1}+2{\rm Cov}\left(\frac{\alpha Y_{t-1}+U_{t-1}}{t},\frac{t-1}{t}Y_{t-1}\right)\\
 & =\frac{\alpha^{2}V_{t-1}+\Sigma}{t^{2}}+\left(\frac{t-1}{t}\right)^{2}V_{t-1}+\frac{2\alpha(t-1)}{t^{2}}V_{t-1}\\
 & =\left[\frac{\alpha^{2}}{t^{2}}+\left(\frac{t-1}{t}\right)^{2}+\frac{2\alpha(t-1)}{t^{2}}\right]V_{t-1}+\frac{1}{t^{2}}\Sigma\\
 & =\left(\frac{t-1+\alpha}{t}\right)^{2}V_{t-1}+\frac{1}{t^{2}}\Sigma.
\end{align*}
We now proceed by induction. The case of $t=1$ follows from the definition
of $Y_{1}$, and the case of $t=2$ follows from the above identity.
Suppose our hypothesis holds at some $t\geq3$. Then, using the identity
$\Gamma(z+1)=z\Gamma(z)$, we have
\begin{align*}
V_{t+1} & =\left(\frac{t+\alpha}{t+1}\right)^{2}V_{t}+\frac{1}{(t+1)^{2}}\Sigma.\\
 & =\left\{\left(\frac{t+\alpha}{t+1}\right)^{2}\left(\frac{1}{t^{2}}+\left[\frac{\Gamma(t+\alpha)}{\Gamma(t+1)}\right]^{2}\sum_{k=1}^{t-1}\left[\frac{\Gamma(k+1)}{k\cdot\Gamma(k+\alpha)}\right]^{2}\right)+\frac{1}{(t+1)^{2}}\right\}\Sigma\\
 & =\left\{\left(\frac{t+\alpha}{t+1}\right)^{2}\frac{1}{t^{2}}+\left[\frac{\Gamma(t+1+\alpha)}{\Gamma(t+2)}\right]^{2}\sum_{k=1}^{t-1}\left[\frac{\Gamma(k+1)}{k\cdot\Gamma(k+\alpha)}\right]^{2}+\frac{1}{(t+1)^{2}}\right\}\Sigma\\
 & =\left\{\left[\frac{\Gamma(t+1+\alpha)}{\Gamma(t+2)}\right]^{2}\sum_{k=1}^{t}\left[\frac{\Gamma(k+1)}{k\cdot\Gamma(k+\alpha)}\right]^{2}+\frac{1}{(t+1)^{2}}\right\}\Sigma.
\end{align*}

\end{proof}

\begin{theorem}
For $t\geq3$, we have that
\begin{equation}
\frac{1}{2}\left[\frac{1}{t}+\frac{1}{t^{2}}+\frac{1}{t^{2(1-\alpha)}}\right]\Sigma\preceq{\rm Var}(Y_{t}(\bar{w}))\preceq4\left[\frac{1}{t}+\frac{1}{t^{2}}+\frac{1}{t^{2(1-\alpha)}}\right]\Sigma.\label{eq:main_var_upper_bd}
\end{equation}
\end{theorem}

\begin{proof}
We first prove the upper bound. Using Lemma~\ref{lem:gautschi},
we first have that 
\begin{align*}
\left[\frac{\Gamma(t+\alpha)}{\Gamma(t+1)}\right]^{2}\sum_{k=1}^{t-1}\left[\frac{\Gamma(k+1)}{k\cdot\Gamma(k+\alpha)}\right] & \leq t^{2(\alpha-1)}\sum_{k=1}^{t-1}\left[\frac{\Gamma(k+1)}{k\cdot\Gamma(k+\alpha)}\right]^{2}\\
 & =t^{2(\alpha-1)}\sum_{k=1}^{t-1}\left[\frac{(k+1)^{1-\alpha}}{k}\right]^{2}\\
 & \leq4t^{2(\alpha-1)}\sum_{k=1}^{t-1}\left[\frac{(k+1)^{1-\alpha}}{k+1}\right]^{2} & \because k\geq\frac{k+1}{2}\\
 & =4t^{2(\alpha-1)}\sum_{k=1}^{t-1}\frac{1}{k^{2\alpha}}.
\end{align*}
Now, since $\alpha\in(0,1)$ and $x\mapsto x^{-2a}$ is monotonically
decreasing, we have 
\[
\sum_{k=1}^{t-1}\frac{1}{k^{2\alpha}}\leq1+\int_{1}^{t-1}\frac{1}{x^{2\alpha}}\ dx=1+\left.\frac{1}{x^{2\alpha+1}}\right]_{1}^{t-1}=\left[1+(t-1)^{-(2\alpha+1)}\right]\leq\left[1+t^{-2\alpha+1}\right].
\]
Combining this with the previous bounds yields
\[
\left[\frac{\Gamma(t+\alpha)}{\Gamma(t+1)}\right]^{2}\sum_{k=1}^{t-1}\left[\frac{\Gamma(k+1)}{k\cdot\Gamma(k+\alpha)}\right]\leq4t^{2(\alpha-1)}\left[1+t^{-2\alpha+1}\right]\leq4\left[\frac{1}{t}+\frac{1}{t^{2(1-\alpha)}}\right],
\]
which gives the upper bound of (\ref{eq:main_var_upper_bd}).

The lower bound has a similar derivation. Using Lemma~\ref{lem:gautschi},
we first have that 
\begin{align*}
\left[\frac{\Gamma(t+\alpha)}{\Gamma(t+1)}\right]^{2}\sum_{k=1}^{t-1}\left[\frac{\Gamma(k+1)}{k\cdot\Gamma(k+\alpha)}\right] & \geq(t+1)^{2(\alpha-1)}\sum_{k=1}^{t-1}\left[\frac{\Gamma(k+1)}{k\cdot\Gamma(k+\alpha)}\right]^{2}\\
 & \geq t^{2(\alpha-1)}\sum_{k=1}^{t-1}\left[\frac{k^{1-\alpha}}{k}\right]^{2}\\
 & =t^{2(\alpha-1)}\sum_{k=1}^{t-1}\frac{1}{k^{2\alpha}}.
\end{align*}
Similarly, since $x\mapsto x^{-2a}$ is monotonically decreasing,
we have
\begin{align*}
\sum_{k=1}^{t-1}\frac{1}{k^{2\alpha}} & =1+\sum_{k=2}^{t-1}\frac{1}{k^{2\alpha}}\geq1+\int_{2}^{t}\frac{1}{x^{2\alpha}}\ dx=1+\left.\frac{1}{x^{2\alpha+1}}\right]_{2}^{t}\\
 & =1+t^{-(2\alpha+1)}-\frac{1}{2^{2\alpha+1}}\geq\frac{1}{2}+t^{-(2\alpha+1)}
\end{align*}
Combining this with the previous bounds yields
\[
\left[\frac{\Gamma(t+\alpha)}{\Gamma(t+1)}\right]^{2}\sum_{k=1}^{t-1}\left[\frac{\Gamma(k+1)}{k\cdot\Gamma(k+\alpha)}\right]\geq t^{2(\alpha-1)}\left[\frac{1}{2}+t^{-(2\alpha+1)}\right]=\frac{1}{t}+\frac{1}{2t^{2(1-\alpha)}}\geq\frac{1}{2}\left[\frac{1}{t}+\frac{1}{t^{2(1-\alpha)}}\right],
\]
which gives the lower bound of (\ref{eq:main_var_upper_bd}).
\end{proof}

\subsection{Improving on Uniform Weighting}\label{sec:improved_weighting}

Let us now construct a better weighting scheme for $\alpha\in(0,1)$
that is continuous in $\alpha$. 
\begin{proposition}
\label{prop:oth_var_id}Suppose $w^{1}=1$ and
\begin{align*}
\hat{w}^{s+1} & =\left(\frac{1}{1+(s-1)(1-\alpha)},\frac{1-\alpha}{1+(s-1)(1-\alpha)},\ldots\frac{1-\alpha}{1+(s-1)(1-\alpha)}\right)\\
\hat{w} & =(\hat{w}^{1},\ldots,\hat{w}^{t})
\end{align*}
for every $s\geq1$. Then, ${\rm Var}(Y_{1}(\hat{w}))=\sigma^{2}$
and, for $t\geq2$, we have
\begin{equation}
{\rm Var}(Y_{t}(\hat{w}))=\left\{ \left(\frac{1-\alpha}{\gamma_{\alpha,t}}\right)^{2}+\sum_{k=2}^{t-1}\left(\frac{C_{\alpha,t-1}}{C_{\alpha,k-1}}\right)^{2}\left(\frac{1-\alpha}{\gamma_{\alpha,k}}\right)^{2}+C_{\alpha,t-1}^{2}\right\} \Sigma\label{eq:oth_var_id}
\end{equation}
where 
\[
\gamma_{\alpha,\ell}:=1+(\ell-1)(1-\alpha),\quad C_{\alpha,\ell}:=\prod_{i=1}^{\ell}\frac{\gamma_{i}+\alpha(1-\alpha)}{\gamma_{i+1}}\quad\forall\ell\geq0.
\]
\end{proposition}

\begin{proof}
Let $t\geq1$ and $\alpha\in[0,1]$ be fixed and denote 
\[
V_{t}={\rm Var}(Y_{t}(\hat{w})),\quad\gamma_{\ell}=\gamma_{\alpha,\ell},\quad C_{\ell}=C_{\alpha,\ell}.
\]
By definition, we have 
\begin{align}
V_{t} & ={\rm Var}\left(\frac{X_{1}}{\gamma_{t}}+\frac{1-\alpha}{\gamma_{t}}\sum_{s=2}^{t-1}X_{s}\right)\nonumber \\
 & ={\rm Var}\left(\frac{X_{1}}{\gamma_{t}}+\frac{1-\alpha}{\gamma_{t}}\sum_{s=2}^{t-1}X_{s}+\frac{1-\alpha}{\gamma_{t}}X_{t}\right)={\rm Var}\left(\frac{\gamma_{t-1}}{\gamma_{t}}Y_{t-1}+\frac{1-\alpha}{\gamma_{t}}X_{t}\right)\nonumber \\
 & ={\rm Var}\left(\frac{\gamma_{t-1}}{\gamma_{t}}Y_{t-1}\right)+{\rm Var}\left(\frac{1-\alpha}{\gamma_{t}}X_{t}\right)+2{\rm Cov}\text{\ensuremath{\left(\frac{\gamma_{t-1}}{\gamma_{t}}Y_{t-1},\frac{1-\alpha}{\gamma_{t}}X_{t}\right)}}\nonumber \\
 & =\left(\frac{\gamma_{t-1}}{\gamma_{t}}\right)^{2}V_{t-1}+\left(\frac{1-\alpha}{\gamma_{t}}\right)^{2}{\rm Var}\left(\alpha Y_{t-1}+U_{t-1}\right)+2{\rm Cov}\text{\ensuremath{\left(\frac{\gamma_{t-1}}{\gamma_{t}}Y_{t},\frac{\alpha[1-\alpha]}{\gamma_{t}}Y_{t-1}\right)}}\nonumber \\
 & =\left[\left(\frac{\gamma_{t-1}}{\gamma_{t}}\right)^{2}+\alpha^{2}\left(\frac{1-\alpha}{\gamma_{t}}\right)^{2}+\frac{2\gamma_{t-1}\alpha(1-\alpha)}{\gamma_{t}^{2}}\right]V_{t-1}+\left(\frac{1-\alpha}{\gamma_{t}}\right)^{2}\Sigma\nonumber \\
 & =\left[\frac{\gamma_{t-1}+\alpha(1-\alpha)}{\gamma_{t}}\right]^{2}V_{t-1}+\left(\frac{1-\alpha}{\gamma_{t}}\right)^{2}\Sigma.\label{eq:oth_recurse}
\end{align}
We now proceed by induction. The case of $t=1$ follows immediately
from the definition of $Y_{1}$, and the case of $t=2$ follows from
the above identity. For $t=3$, we have 
\begin{align*}
V_{t+1} & =\left(\frac{t+\alpha}{t+1}\right)^{2}V_{t}+\frac{1}{(t+1)^{2}}\Sigma.\\
 & =\left\{ \left(\frac{t+\alpha}{t+1}\right)^{2}\left(\frac{1}{t^{2}}+\left[\frac{\Gamma(t+\alpha)}{\Gamma(t+1)}\right]^{2}\sum_{k=1}^{t-1}\left[\frac{\Gamma(k+1)}{k\cdot\Gamma(k+\alpha)}\right]^{2}\right)+\frac{1}{(t+1)^{2}}\right\} \Sigma\\
 & =\left\{ \left(\frac{t+\alpha}{t+1}\right)^{2}\frac{1}{t^{2}}+\left[\frac{\Gamma(t+1+\alpha)}{\Gamma(t+2)}\right]^{2}\sum_{k=1}^{t-1}\left[\frac{\Gamma(k+1)}{k\cdot\Gamma(k+\alpha)}\right]^{2}+\frac{1}{(t+1)^{2}}\right\} \Sigma\\
 & =\left\{ \left[\frac{\Gamma(t+1+\alpha)}{\Gamma(t+2)}\right]^{2}\sum_{k=1}^{t}\left[\frac{\Gamma(k+1)}{k\cdot\Gamma(k+\alpha)}\right]^{2}+\frac{1}{(t+1)^{2}}\right\} \Sigma.
\end{align*}
Now suppose our hypothesis holds for some $t\geq3$ and denote
\[
c_{j}:=\frac{\gamma_{j}+\alpha(1-\alpha)}{\gamma_{j+1}}\implies C_{\alpha,\ell}=\prod_{i=1}^{\ell}c_{j}^{2}
\]
Applying (\ref{eq:oth_recurse}) and using our hypothesis, we have
\begin{align*}
V_{t+1} & =\left[\frac{\gamma_{t}+\alpha(1-\alpha)}{\gamma_{t+1}}\right]^{2}V_{t}+\left(\frac{1-\alpha}{\gamma_{t+1}}\right)^{2}\Sigma=c_{t+1}^{2}V_{t}+\left(\frac{1-\alpha}{\gamma_{t+1}}\right)^{2}\Sigma\\
 & =\left(c_{t}^{2}\left\{ \left[\frac{1-\alpha}{\gamma_{t}}\right]^{2}+\sum_{k=2}^{t-1}\left[\frac{C_{t-1}}{C_{k-1}}\right]^{2}\left[\frac{1-\alpha}{\gamma_{k}}\right]^{2}+C_{t-1}^{2}\right\} +\left(\frac{1-\alpha}{\gamma_{t+1}}\right)^{2}\right)\Sigma.\\
 & =\left(c_{t}^{2}\left[\frac{1-\alpha}{\gamma_{t}}\right]^{2}+\sum_{k=2}^{t-1}\left[\frac{C_{t}}{C_{k-1}}\right]^{2}\left[\frac{1-\alpha}{\gamma_{k}}\right]^{2}+C_{t}^{2}+\left(\frac{1-\alpha}{\gamma_{t+1}}\right)^{2}\right)\Sigma\\
 & =\left(\left(\frac{1-\alpha}{\gamma_{t+1}}\right)^{2}+\sum_{k=2}^{t}\left[\frac{C_{t}}{C_{k-1}}\right]^{2}\left[\frac{1-\alpha}{\gamma_{k}}\right]^{2}+C_{t}^{2}\right)\Sigma.
\end{align*}
\end{proof}



\begin{proposition}
For every $t\geq1$, there exists $\alpha^{*}\in(0,1)$ such that
for any $\alpha\in(\alpha^{*},1]$, it holds that
\[
{\rm Var}(Y_{t}(\hat{w}))\prec{\rm Var}(Y_{t}(\bar{w})).
\]
\end{proposition}

\begin{proof}
For every $\alpha\in[0,1]$, recall that there exists continuous $\hat{f}_{t},\bar{f}_{t}:[0,1]\mapsto(0,\infty)$
such that 
\[
{\rm Var}(Y_{t}(\hat{w}))=\hat{f}_{t}(\alpha)\Sigma,\quad{\rm Var}(Y_{t}(\bar{w}))=\bar{f}_{t}(\alpha)\Sigma.
\]
In particular, we have that 
\[
\hat{f}_{t}(1)=C_{1,t-1}^{2}=1<\sum_{k=1}^{t}\frac{1}{k^{2}}=\bar{f}_{t}(1)
\]
and, hence, that ${\rm Var}(Y_{t}(\hat{w}))\prec{\rm Var}(Y_{t}(\bar{w}))$.
Since $g_{t}(\alpha)=\bar{f}_{t}(\alpha)-\hat{f}_{t}(\alpha)$ is
continuous on $[0,1]$ and positive at $\alpha=1$, the result follows
immediately by continuity.
\end{proof}

\section{PAC Learning Proofs}

\subsection{Hardness Result}\label{sec:hardness}

\begin{theorem}
 Define $c^*_\alpha = \Pr[\forall t \geq 1, W_t \geq 1]$. For any sample size $n \geq 2$ and $d \geq 2$, there exists a distribution $D$, a hypothesis class $F$ with VC dimension $d$, and $f^* \in F$, such that if $A = A_{\ERM}$ then
\[
\liminf_{t \goes \infty} \E[L(f_t)] \ge \frac{c^*_\alpha}{8n}.
\]\end{theorem}
\begin{proof}  
Consider the function class containing thresholds in $1$ dimension. In other words, $X = \mathbb{R}$, and for any $\theta \in \mathbb{R}$, we define $f^\theta(x) \triangleq 1 - 2 \cdot \indic{x \leq \theta}$, and $F = \{f^\theta, -f^\theta \mid \theta \in \mathbb{R}\}$. 

Define $D$ according to the following p.m.f:
\[ 
p(x) = \begin{cases} \frac{1}{2} - \frac{1}{2n} &\text{if x = +1}\\ \frac{1}{n}  &\text{if x = 0}\\ \frac{1}{2} - \frac{1}{2n}  &\text{if x = -1}\end{cases}
\]

On round $t=0$, $\bar{S}_0$ contains $n$ samples drawn from $D$, labeled according to the true concept $f^*$. Let $\bar{S}_{0, \mathcal{X}} $ denote the \emph{set} of unlabeled samples present in the multi-set $\bar{S}_0$. Thus, if $x = 0$ appears multiple times in $\bar{S}_0$, it appears once in $\bar{S}_{0, \mathcal{X}}$ and $\bar{S}_{0, \mathcal{X}} \subset \{-1, 0, +1\}$. We first characterize the probability that we see the samples $x = -1, x = +1$ but not the sample $x = 0$ in round $0$. 

\begin{align}
\Pr[\bar{S}_{0, \mathcal{X}} = \{-1,+1\}] &= \Pr[0 \not\in \bar{S}_{0,\mathcal{X}}] \Pr[\bar{S}_{0, \mathcal{X}} = \{-1,+1\} | 0 \not\in \bar{S}_{0,\mathcal{X}}] \nonumber\\
&= \left(1 - \frac{1}{n}\right)^n \Pr[\bar{S}_{0, \mathcal{X}} = \{-1,+1\} | 0 \not\in \bar{S}_{0,\mathcal{X}}] \nonumber\\
&= \left(1 - \frac{1}{n}\right)^n \left(1 - \Pr[\bar{S}_{0, \mathcal{X}} = \{+1\} \mid 0 \not\in \bar{S}_{0, \mathcal{X}}] - \Pr[\bar{S}_{0, \mathcal{X}} = \{-1\} \mid 0 \not\in \bar{S}_{0, \mathcal{X}}] \right) \nonumber\\
&= \left(1 - \frac{1}{n}\right)^n \left(1 - \frac{1}{2^n} - \frac{1}{2^n} \right) \geq 1/8 \label{eq:zeroth_round}
\end{align}

The third equality follows since, conditioned on the event that the initial sample does not contain any $0$s, it either contains both $\{-1,+1\}$, only $+1$s or only $-1$s. The final equality follows since, conditioned on $0 \not\in \bar{S}_{0,\mathcal{X}}$, samples $+1$ and $-1$ are equally likely. The final inequality follows because the expressions $(1-\frac{1}{n})^n$ and $(1 - \frac{1}{2^{n-1}})$ are both increasing in $n$, and $n \geq 2$ by assumption. 

After round $0$, we will re-index the data to ignore boundaries between rounds. Let $(\tilde{x}_k, \tilde{y}_k)$ denote the $k$th example encountered after round $0$. In other words, writing $k = t n + i$ for $i < n$, $\tilde{x}_k = x^{t+1}_i$. Similarly, let $\tilde{b}_k$ denote whether $\tilde{x}_k$ was labeled by a previous model or nature: $\tilde{b}_k = \tilde{b}_{t n + i} = b_i^{t+1}$. 

We are interested in two statistics, $C_k, \Delta_k$, which we define as: 

$$C_k = \sum_{k' \leq k} \indic{\tilde{x}_k = 0}$$ and 
$$\Delta_k = \sum_{k' \leq k} \indic{\tilde{x}_k = 0} (\indic{\tilde{b}_k = 1} - \indic{\tilde{b}_k = 0}).$$ 

$C_k$ increments every time a sample $x=0$ is encountered, while $\Delta_k$ measures how much more frequently samples $x = 0$ are labeled by a previous model versus the true concept $f^*$. $\{\Delta_k\}$ is a random walk with both a bias and a staying probability. Defining $\Delta_0 = 0$, we have $\Delta_k = \Delta_{k-1}$ with probability $1-1/n$, $\Delta_k = \Delta_{k-1} + 1$ with probability $\alpha/n$, and $\Delta_k = \Delta_{k-1} - 1$ with probability $(1-\alpha)/n$. 

Define $W_l$ as a biased random walk, without any staying probability, independent of $C_k$. $W_0 = 0$, $W_{l+1} = W_l + 1$ with probability $\alpha$ and $W_{l+1} = W_l - 1$ with probability $(1-\alpha)$. It is easy to check that $\Delta_k$ is equal in distribution to $W_{C_k}$, since: 

\begin{align*}
\Pr(W_{C_{k+1}} = W_{C_k}) &= \Pr(C_{k+1} = C_k) = \Pr(\hat{x}_k \not= 0) = 1 - 1/n \\
\Pr(W_{C_{k+1}} = W_{C_k} + 1) &= \Pr(C_{k+1} = C_k + 1, W_{C_k + 1} = W_{C_k} + 1) \\ 
                &= \Pr(C_{k+1} = C_k+1, W_1 = W_0 + 1) \\
                &= \Pr(C_{k+1} = C_k + 1)\Pr(W_1 = W_0 + 1) \\
                &= \Pr(\hat{x}_k = 0)\Pr(W_1 = W_0 + 1) = \alpha / n \\
\Pr(W_{C_{k+1}} = W_{C_k} - 1) &= 1 - \Pr(W_{C_{k+1}} = W_{C_k}) - \Pr(W_{C_{k+1}} = W_{C_k} + 1) = (1 - \alpha) / n
\end{align*}

Now let $\mathcal{E}_{\rm contaminated} = \{\forall k, C_k \geq 1 \implies \Delta_k \geq 1\}$, be the event that once a sample $x = 0$ is encountered (i.e. $C_k \geq 1$), there are more model-labeled examples than naturally-labeled examples for $x = 0$ (i.e. $\Delta_k \geq 1$) for all time. Via the construction $\Delta_k = W_{C_k}$ this happens if and only if $W_l \geq 1$ for all $l \geq 1$. If $W_l \geq 1$ for all $l \geq 1$, then it's clear that $\Delta_k = W_{C_k} \geq 1$ when $C_k \geq 1$. In the other direction, if $W_l < 1$ for some some $l \geq 1$, then there is eventually an index $k$ such that $C_k = l$, and therefore $\Delta_k = W_{C_k} < 1$. Thus,

\begin{equation}
\Pr(\mathcal{E}_{\rm contaminated}) = \Pr(\forall l \geq 1, W_l \geq 1)
\label{eq:random_walk}
\end{equation}

The random variables $\{\tilde{x}_k, \tilde{b}_k\}$ define what examples are encountered subsequent to round $0$ and whether they are labeled by nature or a model. By construction of our setting (i.e. Algorithm~\ref{alg:recursive}), these are independent of the examples $\{x_i^0\}$ encountered in round $0$. $\mathcal{E}_{\rm contaminated}$ is measurable by first set of random variables, and $\bar{S}_{0, \mathcal{X}}$ is measurable by the second set of random variables, and therefore are also independent. 

Hence combining Equations~(\ref{eq:zeroth_round}) and (\ref{eq:random_walk}), we can conclude: 

\begin{equation}
\Pr[\bar{S}_{0, \mathcal{X}} = \{-1,+1\}) \land \mathcal{E}_{\rm contaminated}] \geq \Pr[\forall l \geq 1, W_l \geq 1]/8 \label{eq3}
\end{equation}

So far, we have have not specified algorithm $A_{\ERM}$'s behavior if there are multiple hypotheses with the same empirical risk. Given a dataset that is separable by a threshold, we will define $A_{\ERM}$ as selecting the maximum-margin hypothesis. \footnote{In general, if $\bar{S}_t$ does not contain the example $x = 0$, any tie-breaking rule (including randomized ones) will have a constant probability of mis-classifying it. Treating this with full generality complicates the notation and exposition without providing much additional insight, and so we omit it.} In any other case, including when $\hat{S}$ is not separable, $A_{\ERM}$ can break ties arbitrarily.

Fix $f^* = f^{-1}$, which labels samples $x = 0$ as $+1$. We argue that on the event $\mathcal{E}_{\rm fail} := \{\bar{S}_{0, \mathcal{X}} = \{-1,+1\}\} \land \mathcal{E}_{\rm contaminated}$, the classifiers returned by repeated ERM mis-classify $x = 0$, forever, on every time step. On round $t = 0$, $A_{\mathrm{ERM}}(\bar{S}_0)$ returns a the maximum-margin threshold $f^0$, which assigns $f^0(0) = -1$. On subsequent rounds $t$, it may be that $C_{n t} = 0$, in which case $\bar{S}_t$ still only contains samples $-1,+1$, and $A_{\mathrm{ERM}}$ again selects the maximum-margin classifier. Otherwise, $\Delta_{n t} \geq 1$. By induction, classifiers from all previous rounds mislabeled $x = 0$, and so $\Delta_{n t} \geq 1$ implies that $\bar{S}_t$ contain more examples $(x = 0, y = -1)$ than examples $(x = 0, y = +1)$, thus the empirical risk minimizer is to return a threshold that labels $x = -1$ as $-1$, $x = 0$ as $-1$, and $x = +1$ as $+1$. Therefore, on all rounds $t$, $L(f_t) = 1/n$, and we can conclude that for any $t$:

$$\E[L(f_t)] \geq \E[L(f_t) \mid \mathcal{E}_{\rm fail}]\Pr[\mathcal{E}_{\rm fail}] \geq \frac{c^*_{\alpha}}{8n}$$

Finally, we observe that for any $d$, we can take $\mathcal{X} = \mathbb{R}^d$ and embed the example above in the first coordinate. We take $x_1$ to be drawn according to the p.m.f. $D$, while taking all other coordinates to be $0$ with probability $1$. Consider $F$ to be the family of linear separators in $d$ dimensions: $\mathrm{sgn}(\langle x, w \rangle + b)$ for $x \in \mathbb{R}^d$, $b \in \mathbb{R}$. For any $x$ drawn from this distribution, $\mathrm{sgn}(\langle x, w \rangle + b) = \mathrm{sgn}(x_1 w_1 + b)$, which labels $x_1$ as $-1$ iff $x_1 < -b/w_1$. Thus, for this class too, there exists an $f^*$ such that $\E[R(f_t)] \geq \frac{c^*_{\alpha}}{8n}$. Standard learning theory tells us that $\mathrm{VCD}(F) = d + 1$, concluding the proof.  
\end{proof}

\begin{lemma}
When $\alpha > 1/2$, the event that the random walk $\{W_n\}$ stays positive forever occurs with non-zero probability: $\mathbb{P}[\forall n\geq 1, W_n \geq 1] > 0$
\end{lemma}
\begin{proof}
For $n \geq 1$ define $E_n = \{W_n = 0\}$ as the event that the random walk returns to $0$ after exactly $n$ steps. By the Borel-Cantelli Lemma, if $\sum_{n \geq 1} \mathbb{P}[E_n] < \infty$, then the probability that the events $\{E_n\}$ occur infinitely often goes to zero: $\mathbb{P}[E_n \mathrm{\ i.o}] = 0$. 

Note that on any odd $n$, $W_n \not= 0$ with probability $1$, since the number of leftward steps and rightward steps cannot have the same parity. Therefore:
\begin{align*}
\sum_{n \geq 1} \mathbb{P}[W_n = 0] &= \sum_{n \geq 1} \mathbb{P}[W_{2n} = 0] \\
&= \sum_{n \geq 1} {2n \choose n} \alpha^n (1-\alpha)^n \\
&\leq \sum_{n \geq 1} 4^n \alpha^n (1-\alpha)^n
\end{align*}

Let $\phi = 4 \alpha (1-\alpha)$. As a function of $\alpha \in [0, 1]$, $\phi$ is maximized at $\phi = 1$ when $\alpha = 1/2$. Therefore, $\phi < 1$ under the assumption that $\alpha > 1/2$, and $\sum_{n \geq 1} \mathbb{P}[E_n] \leq \frac{\phi}{1 - \phi} <  \infty$, which implies $\mathbb{P}[E_n\textrm{\ i.o}] = 0$. 

Define $T_n = \inf\{ n \geq 1 \mid W_n = 0\}$, as the first return-time to zero of this random walk. Assume for the sake of contradiction that $\mathbb{P}[T_n < \infty] = 1$, then by the strong Markov property, the random walk visits zero infinitely often with probability $1$, contradicting that $\mathbb{P}[E_n\textrm{\ i.o}] = 0$. Therefore,
$$\mathbb{P}[T_n = \infty] > 0.$$
Next, define $E^+ = \{W_n \mid \forall n \geq 1, W_n > 0\}$ and $E^- = \{W_n \mid \forall n \geq 1, W_n < 0\}$, the events that the random walk stays positive forever or stays negative forever, respectively.
There is a bijection between paths in $E^+$ and $E^-$ attained by reflecting the path over $0$. Furthermore, since paths in $E^-$ must have more left-ward steps than right-ward steps, and $\alpha > 1/2$, $\mathbb{P}(E^+) > \mathbb{P}(E^-)$. 

Finally, since the event $\{T_n = \infty\}$ is the disjoint union of $E^+$ and $E^-$, we have $\mathbb{P}[E^+] \geq \mathbb{P}[T_n = \infty] / 2 > 0$, completing the proof. 
\end{proof}

\subsection{Omitted Calculation in Proof of Lemma \ref{lem:improvement}}\label{sec:update_rule_lemma_calculation}

\begin{lemma}[Chernoff bounds]\footnote{\url{https://en.wikipedia.org/wiki/Chernoff_bound\#Multiplicative_form_(relative_error)}}\label{lem:chernoff}
Let $X_1,...X_m$ be i.i.d. $\mathrm{Bernoulli}(p)$ random variables for $p \in [0,1]$. For $\gamma \in [0,1]$
$$\Pr \left[\sum_{i=1}^m X_i \leq (1-\gamma)mp\right] \leq \exp(-\gamma^2mp/2).$$
\end{lemma}

\paragraph{Omitted calculation.} In the setting of Proof of Lemma \ref{lem:improvement}, Case 2: $L(g_k) > \varepsilon_{k+1}$, our goal is to show that $P_k^\oplus$ is sufficiently large. We let $X_i$ denote the indicator variable that is $1$ if and only if the $i$th example of $T_k^\oplus$, $(x_i,y^\oplus_i)$, satisfies $y^\oplus_i=1$. This occurs with probability $L(g_k)(1-\alpha)$ independently, and we draw the first $m$ samples for $P_k^\oplus$. Hence $|P_k^\oplus| = \sum_{i=1}^m X_i$. We have
\begin{align*}
    \Pr \left[|P_k^\oplus| \leq \tfrac 1 2 m\varepsilon_{k+1}(1-\alpha)\right] =&\Pr \left[\sum_{i=1}^m X_i \leq \tfrac 1 2 m\varepsilon_{k+1}(1-\alpha)\right] \\
    \leq& \Pr \left[\sum_{i=1}^m X_i \leq \tfrac 1 2 mL(g_k)(1-\alpha)\right] &\quad\text{(since $L(g_k) > \varepsilon_{k+1}$)}\\
    \leq& \exp(-m(1-\alpha) L(g_k)/8) &\quad\text{(by Lemma \ref{lem:chernoff} using $\gamma=\tfrac 1 2$)}\\
    \leq& \exp(-m(1-\alpha) \varepsilon_{k+1}/8) &\quad\text{(since $L(g_k) > \varepsilon_{k+1}$)}\\
\end{align*}
Therefore taking
$$m = m_P(\varepsilon_{k+1}, \tfrac \delta 2) \cdot \frac  {8\log\left(\tfrac {1} {(\delta/2)}\right)} {\alpha \varepsilon_{k+1}}$$
as in Lemma \ref{lem:improvement} ensures that $|P_k^\oplus| > 4\log(\tfrac 2 \delta)\cdot m_P(\varepsilon_{k+1},\tfrac \delta 2)$ with probability $\geq 1 - \exp(-m(\varepsilon_{k+1},\tfrac \delta 2)\cdot \log(\tfrac 1 {\delta/2})) \geq 1 -\tfrac \delta 2$.

\end{document}